\documentclass{cmslatex}
\usepackage[paperwidth=7in, paperheight=10in, margin=.875in]{geometry}
\usepackage{amsmath}
\usepackage{amsfonts}
\usepackage{subfig}
\usepackage{latexsym}
\usepackage{graphicx}
\usepackage{amssymb}
\usepackage{algorithmic}

\usepackage{enumerate}
\usepackage{bbm}
\usepackage{dsfont}
\usepackage{comment}
\usepackage[utf8]{inputenc}
\usepackage{color}
\usepackage{multicol}
\usepackage{placeins}
\definecolor{darkgreen}{rgb}{0,0.7,0}

\definecolor{darkblue}{rgb}{0,0,0.7}

\newcommand{\R}{\mathbb{R}}

\newcommand{\D}{\mathcal{D}}

\newcommand{\1}{\mathds{1}}

\newcommand{\dkl}{D_{\mbox {\tiny{\rm KL}}}}

\newcommand{\E}{\mathbb{E}}

\newcommand{\envelope}{(\raisebox{-.5pt}{\scalebox{1.45}{\Letter}}\kern-1.7pt)}

\DeclareMathOperator*{\argmin}{arg\,min}

\definecolor{mygreen}{rgb}{0.1,0.75,0.2}

\newcommand{\nc}{\normalcolor}

\usepackage{appendix}
\usepackage{algorithm}

\renewcommand{\theequation}{\arabic{section}.\arabic{equation}}

\newtheorem{thm}{Theorem}[section]

\newtheorem{remark}[thm]{Remark}

\begin{document}
\title{Variational Characterizations of Local Entropy and Heat Regularization in Deep Learning}


\author{N. Garc\'{i}a Trillos \thanks{Department of Statistics, University of Wisconsin Madison.} \and Z. Kaplan \thanks{Division of Applied Mathematics, Brown University.}
\and{D. Sanz-Alonso}  \thanks{Department of Statistics, University of Chicago.}}




\pagestyle{myheadings} \markboth{On Local Entropy and Heat Regularization in Deep Learning}{N. Garc\'{i}a Trillos, Z. Kaplan, and D. Sanz-Alonso} \maketitle


\begin{abstract}
The aim of this paper is to provide new theoretical and computational understanding on two loss regularizations employed in deep learning, known as local entropy and heat regularization. For both regularized losses we introduce variational characterizations that naturally suggest a two-step scheme for their optimization, based on the iterative shift of a probability density and the calculation of a best  Gaussian approximation in Kullback-Leibler divergence. Under this unified light, the optimization schemes for local entropy and heat regularized loss differ only over which argument of the Kullback-Leibler divergence is used to find the best Gaussian approximation. Local entropy corresponds to minimizing over the second argument, and the solution is given by moment matching. This allows to replace traditional back-propagation calculation of gradients by sampling algorithms, opening an avenue for gradient-free, parallelizable training of neural networks.
\end{abstract}


\section{Introduction}\label{sec:intro}
The development and assessment of optimization methods for the training of deep neural networks has brought forward novel questions that call for new theoretical insights and computational techniques \cite{bottou2016optimization}. The performance of a network is determined by its ability to generalize, and choosing the network parameters by finding the global minimizer of the loss may be not only unfeasible, but also undesirable. In fact, training to a prescribed accuracy with competing optimization schemes may lead, consistenly, to different generalization error \cite{keskar2016large}. A possible explanation is that parameters in flat local minima of the loss give better generalization  \cite{hochreiter1997flat},  \cite{keskar2016large}, \cite{chaudhari2017deep}, \cite{entropysgd} and that certain schemes favor convergence to wide valleys of the loss function. These observations have led to the design of  algorithms  that employ gradient descent on a regularized loss, actively seeking minima located in wide valleys of the original loss \cite{entropysgd}. While it has been  demonstrated that the flatness of minima cannot fully explain generalization in deep learning \cite{dinh2017sharp}, \cite{neyshabur2017exploring}, there are various heuristic \cite{baldassi2015subdominant}, theoretical  \cite{chaudhari2017deep}, and empirical \cite{entropysgd} arguments that support regularizing the loss. In this paper we aim to provide new understanding on two such regularizations, referred to as local entropy and
heat regularization. 

Our first contribution is to introduce variational characterizations for both regularized loss functions. These characterizations, drawn from the literature on large deviations \cite{dupuis2011weak}, naturally suggest a two-step scheme for their optimization, based on the iterative shift of a probability density and the calculation of a best Gaussian approximation in Kullback-Leibler divergence. The schemes for both regularized losses differ only over the argument of the (asymmetric) Kullback-Leibler divergence that they minimize. Local entropy minimizes over the second argument, and the solution is given by moment matching; heat regularization minimizes over the first argument, and its solution is defined implicitly.

The second contribution of this paper is to investigate some theoretical and computational implications of the variational characterizations. On the theoretical side, we prove that if the best Kullback-Leibler approximations could be computed exactly, then the regularized losses are monotonically decreasing along the sequence of optimization iterates. This monotonic behavior suggests that the two-step iterative optimization schemes have the potential of being stable provided that the Kullback-Leibler minimizers can be computed accurately. On the computational side, we show that the two-step iterative optimization of local entropy agrees with gradient descent on the regularized loss provided that the learning rate matches the regularization parameter. Thus, the two-step iterative optimization of local entropy computes gradients implicitly in terms of expected values; this observation opens an avenue for gradient-free, paralelizable training of neural networks based on sampling. In contrast, the scheme for heat regularization finds the best Kullback-Leibler Gaussian approximation over the first argument, and its computation via stochastic optimization \cite{robbins1956empirical}, \cite{pinski2015algorithms} involves evaluation of gradients of the original loss. \nc

Finally, our third contribution is to perform a numerical case-study to assess the performance of various implementations of the two-step iterative optimization of local entropy and heat regularized functionals. These implementations differ in how the minimization of Kullback-Leibler is computed and the argument that is minimized. Our experiments suggest, on the one hand, that the computational overload of the regularized methods far exceeds the cost of performing stochastic gradient descent on the original loss. On the other hand, they also suggest that for moderate-size architectures ---where the best  Kullback-Leibler Gaussian approximations can be computed effectively--- the generalization error with regularized losses is more stable than for stochastic gradient descent over the original loss. For this reason, we investigate using stochastic gradient descent on the original loss for the first parameter updates, and then switching to optimize over a regularized loss. We also investigate numerically the choice and scoping of the regularization parameter. Our understanding upon conducting thorough numerical experiments is that, while sampling-based optimization of local entropy has the potential of being practical if parallelization is exploited and back-propagation gradient calculations are expensive, existing implementations of regularized methods in standard architectures are more expensive than stochastic gradient descent and do not clearly outperform it. \nc

Several research directions stem from this work. A broad one is to explore the use of  local entropy and heat regularizations in complex optimization problems outside of deep learning, e.g. in the computation of maximum a posteriori estimates in high dimensional Bayesian inverse problems. A more concrete direction is to generalize the Gaussian approximations within our two-step iterative schemes and allow to update both the mean and covariance of the Gaussian measures. 

The rest of the paper is organized as follows. Section \ref{sec:background} provides background on optimization problems arising in deep learning, and reviews various analytical and statistical interpretations of local entropy and heat regularized losses. In Section \ref{sec:variationallocalentropy} we introduce the variational characterization of local entropy, and derive from it a two-step iterative optimization scheme. 
Section \ref{sec:regularizef} contains analogous developments for heat regularization. Our presentation in Section \ref{sec:regularizef} is parallel to that in Section \ref{sec:variationallocalentropy}, as we aim to showcase the unity that comes from the variational characterizations of both loss functions. Section \ref{sec:kullback-leibler} reviews various algorithms for Kullback-Leibler minimization, and we conclude in Section \ref{sec:experiments} with a numerical case study.

\section{Background}\label{sec:background}

Neural networks are revolutionizing numerous fields including image and speech recognition, language processing, and robotics \cite{lecun2015deep}, \cite{goodfellow2016deep}. Broadly, neural networks are parametric families of functions used to assign outputs to inputs. The parameters $x\in \R^d$ of a network are chosen by solving a non-convex optimization problem of the form
\begin{equation}\label{eq:optimization}
\argmin_x f(x) = \argmin_x \frac1N \sum_{i=1}^N f_i(x),
\end{equation}
where each $f_i$ is a loss associated with a training example. Most popular training methods employ backpropagation (i.e. automatic differentiation) to perform some variant of gradient descent over the loss $f$. In practice, gradients are approximated using a random subsample of the training data known as \emph{minibatch}. Importantly, accurate solution of the optimization problem \eqref{eq:optimization} is \emph{not} the end-goal of neural networks; their performance is rather determined by their generalization or \emph{testing} error, that is, by their ability to accurately assign outputs to unseen examples. 

A substantial body of literature \cite{entropysgd}, \cite{neyshabur2017exploring}, \cite{bottou2016optimization} has demonstrated that optimization procedures with similar training error may consistently lead to different testing error. For instance, large mini-batch sizes have been shown to result in poor generalization \cite{keskar2016large}. Several explanations have been set forth, including overfitting, attraction to saddle points and explorative properties \cite{keskar2016large}. A commonly accepted theory is that flat local minima of the loss $f$ leads to better generalization than sharp minima \cite{hochreiter1997flat}, \cite{entropysgd}, \cite{keskar2016large}, \cite{chaudhari2017deep}. As noted in \cite{dinh2017sharp} and \cite{neyshabur2017exploring} this explanation is not fully convincing, as due to the high number of symmetries in deep networks one can typically find many parameters that have different flatness but define the same network. Further, reparameterization may alter the flatness of minima. While a complete understanding is missing, the observations above have prompted the development of new algorithms that actively seek minima in wide valleys of the loss $f.$ In this paper we provide new insights on potential advantages of two such approaches, based on local-entropy and heat regularization.

\subsection{Background on Local-Entropy Regularizatoin}
We will first study optimization of networks performed on a regularization of the loss $f$ known as \emph{local entropy}, given by 
\begin{equation}\label{def:localentropy}
 F_\tau(x):= - \log \left( \int_{\R^d} \exp\bigl( -f(x')\bigr) \varphi_{x,\tau} (x') dx' \right),
\end{equation}
where here and throughout $\varphi_{x,\tau}$ denotes the Gaussian density in $\mathbb{R}^d$ with mean $x$ and variance $\tau I.$ For given $\tau,$  $F_\tau(x)$ averages values of $f$ focusing on a neighborhood of size $\tau.$ Thus, for $F_\tau(x)$ to be small it is required that $f$ is small throughout a $\tau$-neighborhood of $x.$ Note that $F_\tau$ is equivalent to $f$ as $\tau\to 0,$ and becomes constant as $\tau\to \infty.$ Figure \ref{LER_illustration} shows that local entropy flattens sharp isolated minima, and deepens wider minima. 

A natural statistical interpretation of minimizing the loss $f$ is in terms of {\em maximum likelihood} estimation. Given training data $ {\cal D}$ one may define the likelihood function 
\begin{equation}\label{eq:likelihood}
 \rho_f(x|{\cal D}) \propto \exp\bigl(-f(x)\bigr).
\end{equation}
Thus, minimizing $f$ corresponds to maximizing the likelihood $\rho_f.$ In what follows we assume that $\rho_f$ is normalized to integrate to $1.$ Minimization of local entropy can also be interpreted in statistical terms, now as computing a {\em maximum marginal likelihood}. Consider a Gaussian prior distribution $p(x'|x) = \varphi_{x,\tau} (x'),$ indexed by a hyperparameter $x,$  on the parameters $x'$ of the neural network. Moreover, assume a likelihood $p(x' | \D) \propto\exp\bigl(-f(x')\bigr)$ as in equation \eqref{eq:likelihood}. Then, minimizing local entropy corresponds to maximizing the marginal likelihood

\begin{align}\label{eq:marginallikelihood}
\begin{split}
p(\D | x) &= \int p(\D | x') p(x'|x) dx' \\
              & = \int \exp\bigl(-f(x')\bigr) \, \varphi_{x,\tau} (x') \, dx'.
\end{split}
\end{align}


We remark that the right-hand side of equation \eqref{eq:marginallikelihood} is the convolution of the likelihood $\rho_f$ with a Gaussian, and so we have 
\begin{equation}\label{eq:convolution}
F_\tau(x) \propto -\log\Bigl( \rho_f \ast \varphi_{0,\tau}(x) \Bigr).
\end{equation}
Thus,  local entropy $F_\tau$  can be interpreted as a regularization of the likelihood $\rho_f.$ 

\begin{figure}
\includegraphics[width=\linewidth]{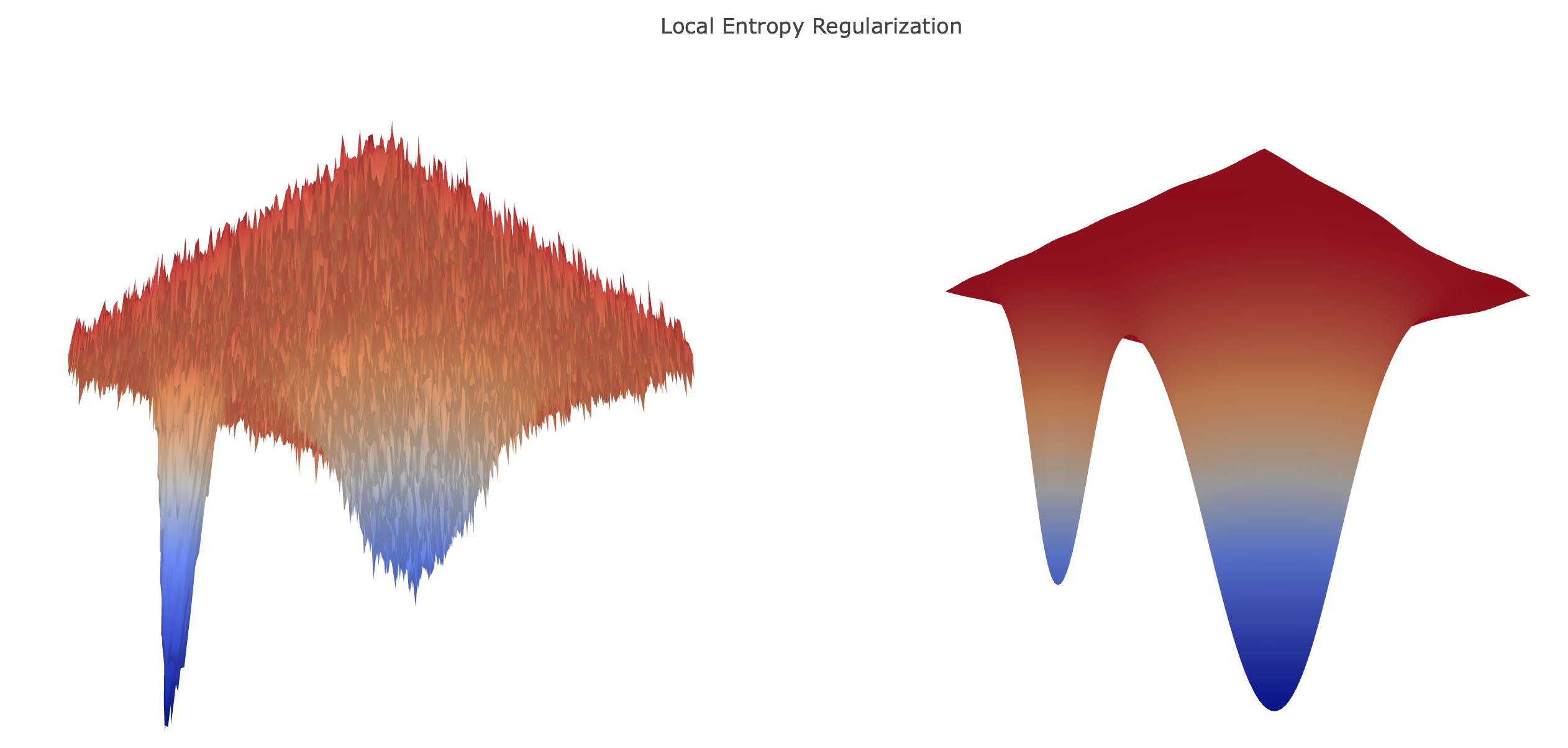}
\caption{Toy example of local entropy regularization for a two dimensional lost function. Note how the wider minima from the left figure deepens on the right, while the sharp minima become relatively shallower.}
\label{LER_illustration}
\end{figure}

\subsection{Background on Heat Regularization}
We will also consider smoothing of the loss $f$ through the heat regularization, defined by
$$F^H_\tau(x) := \int_{\R^d} f(x') \varphi_{x,\tau}(x') \, dx'. $$
Note that $F^H_\tau$ regularizes the loss $f$ directly, rather than the likelihood $\rho_f$:
$$ F^H(x) = f\ast \varphi_{0,\tau} (x). $$
Local entropy and heat regularization are, clearly, rather different. Figure \ref{HR_illustration} shows that while heat regularization smooths the energy landscape, the relative macroscopic depth of local minima is marginally modified.  
Our paper highlights, however, the common underlying structure of the resulting optimization problems.
Further analytical insights on both regularizations in terms of partial differential equations and optimal control can be found in \cite{chaudhari2017deep}. 
\subsection{Notation}
For any $x\in \R^d$ and $\tau>0$ we define the probability density 
\begin{equation}\label{eq:qkdensity}
q_{x,\tau}(x') := \frac{1}{Z_{x,\tau}} \exp\Bigl( -f(x') - \frac{1}{2\tau} \lvert x-x' \rvert^2 \Bigr), 
\end{equation}
where $Z_{x,\tau}$ is a normalization constant. These densities will play an important role throughout. 

We denote the Kullback-Leibler divergence between densities $p$ and $q$ in $\mathbb{R}^d$  by
\begin{equation}
\dkl(p \| q) := \int_{\mathbb{R}^d}  \log \left(\frac{p(x)}{q(x)} \right) p(x) \, dx.
\end{equation}
Kullback-Leibler is a divergence in that $\dkl(p \| q)\ge 0,$ with equality iff $p = q.$ However, the Kullback-Leibler is not a distance as in particular it is not symmetric;  this fact will be relevant in the rest of this paper.

\begin{figure}
\includegraphics[width=\linewidth]{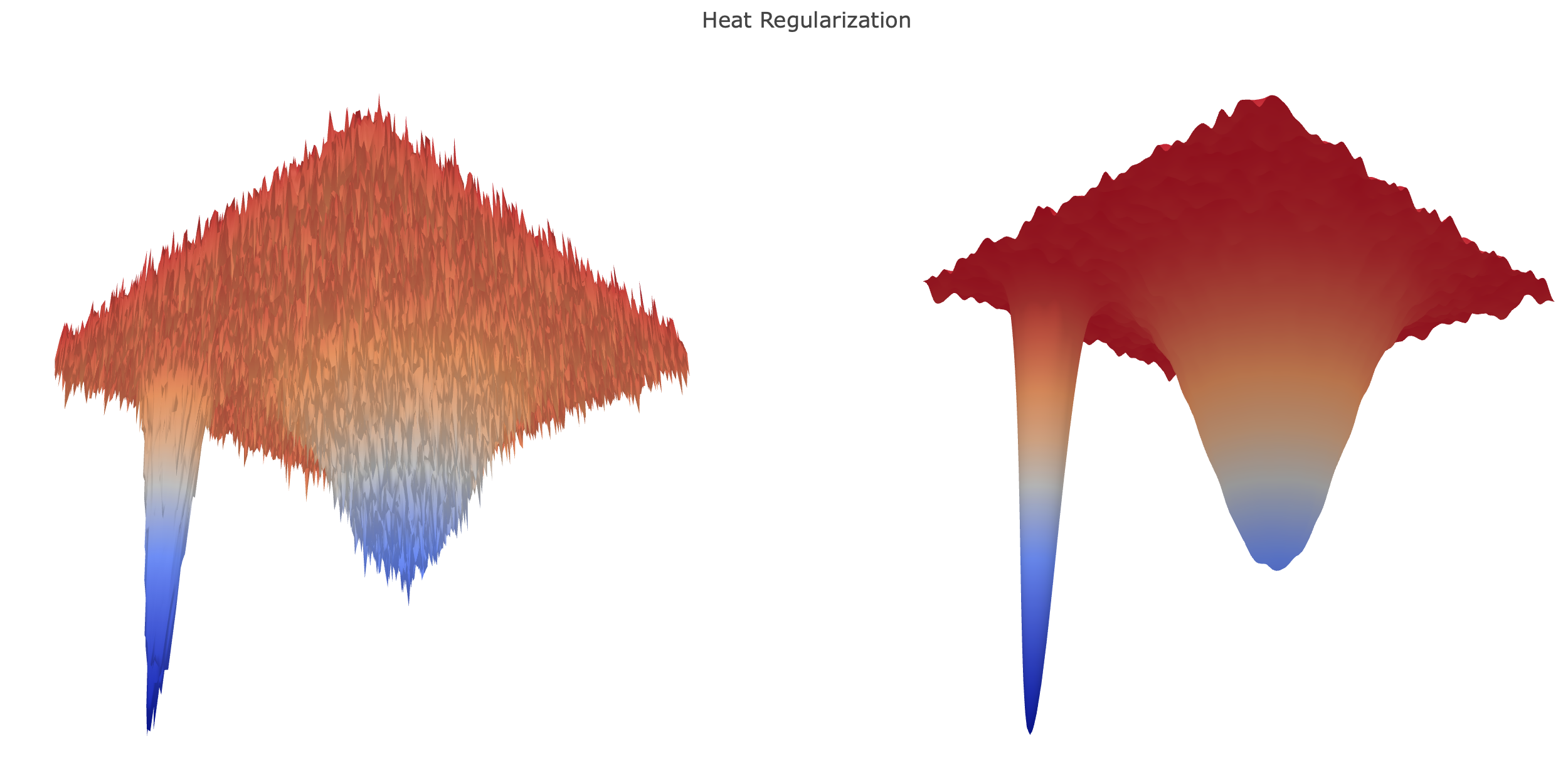}
\caption{Toy example of heat regularization for a two dimensional loss function. Here the smoothing via convolution with a Gaussian amounts to a blur, altering the texture of the landscape without changing the location of deep minima. }
\label{HR_illustration}
\end{figure}

\section{Local Entropy: Variational Characterization and Optimization}\label{sec:variationallocalentropy}

In this section we introduce a variational characterization of local entropy. We will employ this characterization  to derive a monotonic algorithm for its minimization.
The following result is well known in large deviation theory \cite{dupuis2011weak}. We present its proof for completeness. 

\begin{theorem} \label{th:chacacterization}
The local entropy admits the following variational characterization:
\begin{align}
\begin{split}
 F_\tau(x):&= - \log \left( \int_{\R^d} \exp\bigl( -f(x')\bigr) \varphi_{x,\tau} (x') dx' \right)  \\
 & = \min_{q} \left\{ \int_{\R^d} f(x') q(x')\, dx' + \dkl (q \| \varphi_{x,\tau}) \right\}.\label{eq:char} 
 \end{split}
\end{align}
Moreover, the density $q_{x,\tau}$ defined in equation \eqref{eq:qkdensity} achieves the minimum in  \eqref{eq:char}.
\end{theorem}

\begin{proof}
For any density $q,$
\begin{equation}\label{eq:dkleq}
\dkl(q \| q_{x,\tau}) =  \int_{\R^d} f(x') q(x')\, dx' +  \dkl(q \| \varphi_{x,\tau}) + \log(Z_{x,\tau}).  
\end{equation}
Hence,
\begin{align*}
q_{x,\tau} &= \argmin_q{ \dkl( q \| q_{x,\tau} )} \\
& = \argmin\left\{ \int_{\R^d} f(x') q(x')\, dx' +  \dkl(q \| \varphi_{x,\tau}) + \log(Z_{x,\tau})\right\} \\
& = \argmin_q \left\{ \int_{\R^d} f(x') q(x') \, dx' +  \dkl(q \| \varphi_{x,\tau}) \right\},
\end{align*}
showing that $q_{x,\tau}$ achieves the minimum.
To conclude note that  $F_\tau(x) = -\log(Z_{x,\tau}),$ and so taking minimum over $q$ on both sides of equation \eqref{eq:dkleq} and rearranging gives equation \eqref{eq:char}.
\end{proof}

\subsection{Two-step Iterative Optimization}
From the variational characterization \eqref{eq:char} it follows that
\begin{equation}\label{eq:vartoalgo}
\argmin_x F_\tau(x) =
\argmin_{x} \, \min_{q} \left\{ \int_{\R^d} f(x') q(x')\, dx' + \dkl (q \| \varphi_{x,\tau}) \right\}.
\end{equation}
Thus, a natural iterative approach to finding the minimizer of $F_\tau$ is to alternate between i) minimization of the term in curly brackets over densities $q;$ and ii) finding the associated minimizer over $x.$ For the former  we can employ the explicit formula given by equation \eqref{eq:qkdensity}, while for the latter we note that the integral term does not depend on the variable $x,$ and that the minimizer of the map 
\[ x \mapsto \dkl (q_{x_k,\tau} \| \varphi_{x,\tau}) \]
is unique, and given by the expected value of $q_{x_k,\tau}.$ The statistical interpretation of these two steps is perhaps most natural  through the variational formulation of the Bayesian update \cite{BANGTSA}:  the first step finds a posterior distribution associated with likelihood $\rho_f \propto \exp(-f)$ and prior $\varphi_{x,\tau};$ the second computes the posterior expectation, which is used to define the prior mean in the next iteration. It is worth noting the parallel between this two-step optimization procedure and the empirical Bayes interpretation of local entropy mentioned in Section \ref{sec:background}.

In short, the expression \eqref{eq:vartoalgo}  suggests the following simple scheme for minimizing local-entropy:
\FloatBarrier
\begin{algorithm}
\begin{algorithmic}
\STATE Choose $x_0 \in \R^d$ and for $k=0, \ldots, K-1$ do:\label{algorithm1}
\begin{enumerate}
\item Define $q_{x_k,\tau}$ as in equation \eqref{eq:qkdensity}.
\item Define $x_{k+1}$ as the minimizer, $\E_{X \sim q_{x_k,\tau} } \left( X \right),$ of the map
\[ x \mapsto \dkl(q_{x_k,\tau}   \| \varphi_{x,\tau}). \]
\end{enumerate}
\end{algorithmic}
\caption{}
\end{algorithm}
\FloatBarrier

In practice, the expectation in the second step needs to be approximated. We will explore the potential use of gradient-free sampling schemes in Subsection \ref{ssec:minalg2} and in our numerical experiments.

A seemingly unrelated approach to minimizing the local entropy $F_\tau$ is to employ gradient descent and set
\begin{equation}
x_{k+1} = x_k - \eta \nabla F_\tau(x_k),
\label{graddescent}
\end{equation}
where $\eta$ is a learning rate. We now show that the iterates $\{x_k\}_{k=0}^K$ given by Algorithm \ref{algorithm1} agree with those given by gradient descent with learning rate $\eta = \tau.$

By direct computation
\[ \nabla F_\tau(x) = \frac{1}{\tau}\Bigl( x - \E_{X \sim q_{x,\tau} }( X) \Bigr). \]
Therefore, 
\begin{equation}
\nabla F_\tau(x_k) = \frac{1}{\tau}\left( x_k - \E_{X \sim q_{x_k,\tau} }( X) \right) = \frac{1}{\tau}( x_k - x_{k+1} ),
\label{gradient}
\end{equation}
establishing that Algorithm \ref{algorithm1} performs gradient descent with learning rate $\tau$.  This choice of learning rate leads to monotonic decrease of local entropy, as we show in the next subsection.

\begin{remark}
\label{rem:etatau} 
In this paper we restrict our attention to the update scheme \eqref{graddescent} with $\eta = \tau$. For this choice of learning rate we can deduce theoretical monotonicity according to Theorem \ref{theoremmonotone} below, but it may be computationally advantageous to use $\eta \not = \tau$ as explored in \cite{entropysgd}.  
\end{remark}


\subsection{Majorization-Minorization and Monotonicity}\label{ssec:majorization}
We now show that Algorithm \ref{algorithm1} is a majorization-minimization algorithm. Let
\[ A(x, \tilde{x}):= \int_{\R^d} f(x') q_{\tilde{x},\tau}(x')\, dx' + \dkl (q_{\tilde{x},\tau} \| \varphi_{x,\tau}) , \]
where $q_{\tilde{x},\tau}$ is as in \eqref{eq:qkdensity}. It follows that $A(x,x) = F_\tau(x)$ for all $x \in \R^d$ and  that $A(x, \tilde{x}) \geq F_\tau(x)$ for arbitrary $x, \tilde{x}$; in other words, $A$ is a majorizer for $F_\tau$. In addition, it is easy to check that the updates
\[ x_{k+1} = \argmin_{x} A(x,x_k) \]
coincide with the updates in Algorithm \ref{algorithm1}.
As a consequence we have the following theorem.
\begin{theorem}(Monotonicity and stationarity of Algorithm \ref{algorithm1}) \label{theoremmonotone}
The sequence $\{x_k\}_{k=0}^{K}$ generated by Algorithm \ref{algorithm1} satisfies
\[ F_\tau(x_{k}) \le F_\tau(x_{k-1}) , \quad 1 \le k \le K. \]
Moreover, equality holds only when $x_k$ is a critical point of $F_\tau$.
\end{theorem}
\begin{proof}
The monotonicity follows immediately from the fact that our algorithm can be interpreted as a majorization-minimization scheme. For the stationarity note that equation \eqref{gradient} shows that $x_k= x_{k+1}$ if and only if $\nabla F_\tau(x_k)=0$.
\end{proof}

\section{Heat Regularization: Variational Characterization and Optimization}\label{sec:regularizef}
In this section we consider direct regularization of the loss function $f$ as opposed to regularization of the density function $\rho_f.$ The following result is analogous to Theorem \ref{th:chacacterization}. Its proof is similar and hence omitted. 

\begin{theorem} The heat regularization $F^H_\tau$ admits the following variational characterization:
\begin{align}\label{eq:variationalheat}
\begin{split}
F^H_\tau(x) :&= \int_{\R^d} f(x') \varphi_{x,\tau}(x') \, dx' \\
&=\min_{q} \left\{ \log \left( \int_{\R^d} \exp(f(x')) q(x')\, dx' \right) + \dkl(\varphi_{x,\tau} \| q) \right\}. 
\end{split}
\end{align}
Moreover, the density $q_{x,\tau}$ defined in equation \eqref{eq:qkdensity} achieves the minimum in  \eqref{eq:variationalheat}.
\end{theorem}

\subsection{Two-step Iterative Optimization}\label{ssec:heatalgorithm}
From equation \eqref{eq:variationalheat} it follows that
\begin{equation}
\argmin_{x} F^H_\tau(x) =\argmin_{x} \, \inf_{q} \left\{ \log \left( \int_{\R^d} \exp(f(x')) q(x') \, dx' \right) + \dkl(\varphi_{x,\tau} \| q) \right\}.
\label{Swapped}
\end{equation}
In complete analogy with Section \ref{sec:variationallocalentropy}, equation \eqref{Swapped} suggests the following optimization scheme to  minimize $F^H_\tau.$

\FloatBarrier
\begin{algorithm}
\begin{algorithmic} 
\STATE Choose $x_0 \in \R^d$ and for $k=0, \ldots, K-1$ do:
\begin{enumerate}
\item  Define $q_{x_k,\tau}$ as in equation \eqref{eq:qkdensity}.
\item Define $x_{k+1}$ by minimizing the map
\[ x \mapsto \dkl(\varphi_{x,\tau} \| q_{x_k,\tau}). \]
\end{enumerate}
\end{algorithmic}
\caption{ \label{algorithm2}}
\end{algorithm}
\FloatBarrier

The key difference with Algorithm \ref{algorithm1} is that the arguments of the Kullback-Leibler divergence are reversed. While  $x \mapsto \dkl(q_{x_k,\tau} \| \varphi_{x,\tau})$ has a unique minimizer given by $\mathbb{E}_{X\sim q_{x_k,\tau}}(X),$ minimizers of $x \mapsto \dkl(\varphi_{x,\tau} \| q_{x_k,\tau})$ need not be unique. Moreover, the latter minimization is implicitly defined via an expectation and its computation via a Robbins-Monro \cite{robbins1956empirical} approach requires repeated evaluation of the gradient of $f$. We will outline the practical implementation of this minimization in Section \ref{ssec:minalg2}.

\subsection{Majorization-Minorization and Monotonicity}
As in subsection \ref{ssec:majorization} it is easy to see that
\[ A^H(x, \tilde{x}):= \log \left( \int_{\R^d} \exp(f(x')) q_{\tilde{x},\tau}(x') \, dx' \right) + \dkl(\varphi_{x,\tau} \| q_{\tilde x,\tau}) \]
is  a majorizer for $F^H_\tau$. This can be used to show the following theorem, whose proof is identical to that of Theorem \ref{theoremmonotone} and therefore omitted. 
\begin{theorem}(Monotonicity of Algorithm \ref{algorithm2}) \label{theoremmonotone2}
The sequence $\{x_k\}_{k=0}^{K}$ generated by Algorithm \ref{algorithm2} satisfies
\[ F^H_\tau(x_{k}) \le F^H_\tau(x_{k-1}) , \quad 1 \le k \le K. \]
\end{theorem}

%
%
%
%
%

\section{Gaussian Kullback-Leibler Minimization}\label{sec:kullback-leibler}

In Sections \ref{sec:variationallocalentropy} and \ref{sec:regularizef} we considered the local entropy $F_\tau$ and heat regularized loss $F_\tau^H$ and introduced two-step iterative optimization schemes for both loss functions. We summarize these schemes here for comparison purposes: 

	\begin{multicols}{2}
		{\bf Optimization of $F_\tau$}\\
		Let $x_0 \in \R^d$ and for $k=0, \ldots, K-1$ do:
		\begin{enumerate}
			\item  Define $q_{x_k,\tau}$ as in equation \eqref{eq:qkdensity}.
			\item Let $x_{k+1}$ be the minimizer of
			\[ x  \mapsto \dkl(q_{x_k,\tau}   \| \varphi_{x,\tau}). \]
		\end{enumerate}
		\columnbreak
		{\bf Optimization of $F_\tau^H$} \\
		Let $x_0 \in \R^d$ and for $k=0, \ldots, K-1$ do:
		\begin{enumerate}
			\item  Define $q_{x_k,\tau}$ as in equation \eqref{eq:qkdensity}.
			\item Let $x_{k+1}$ be a minimizer of
			\[ x  \mapsto \dkl(\varphi_{x,\tau} \| q_{x_k,\tau}). \]
		\end{enumerate}
	\end{multicols}
	
Both schemes involve finding, at each iteration, the mean vector that gives the best approximation, in Kullback-Leibler, to a probability density. For local entropy the minimization is with respect to the second argument of the Kullback-Leibler divergence, while for heat regularization the minimization is with respect to the first argument. It is useful to compare, in intuitive terms, the two different minimization
problems, both leading to a ``best Gaussian''.
In what follows we drop the subscripts and use the following nomenclature:
\begin{align*}
&\dkl(q||\varphi) = \E^{q}{\left[\log\left(\frac{q}{p}\right)\right]} \quad \quad \,\, \text{``Mean-seeking"} \\
&\dkl(\varphi|| q) = \E^{\varphi}{\left[\log\left(\frac{\varphi}{q}\right)\right]} \quad \quad  \text{``Mode-seeking"}.\\
\end{align*}

Note that in order to minimize $\dkl(\varphi\|q)$ we need $\log{\frac{\varphi}{q}}$ to be small over the support of $\varphi,$ which can happen when $\varphi \simeq q$ or $\varphi \ll q$. This illustrates the fact that minimizing $\dkl(\varphi \|q)$ may miss out components of $q$. 
For example, in Figure \ref{KL_illustration} left panel  $q$ is a bi-modal like distribution 
but minimizing $\dkl(\varphi||q)$ over Gaussians $\varphi$ can only give a single mode 
approximation which is achieved by matching one of the modes (minimizers are not guaranteed to be unique);
we may think of this as  ``mode-seeking''. In contrast, when minimizing 
$\dkl(q\|\varphi)$ over Gaussians $\varphi$ we want $\log{\frac{q}{\varphi}}$ to be small where $\varphi$ appears as the denominator. This implies that wherever $q$ has some mass we must let $\varphi$ also have some mass there in order to keep $\frac{q}{\varphi}$ as close as possible to one. Therefore the minimization is carried out by allocating the mass of $\varphi$ in a way such that on average the discrepancy between $\varphi$ and $q$ is minimized, as shown in Figure ~\ref{KL_illustration} right panel; hence the label ``mean-seeking.''

\begin{figure}
\begin{center}
\includegraphics[height = 3cm,width=10cm]{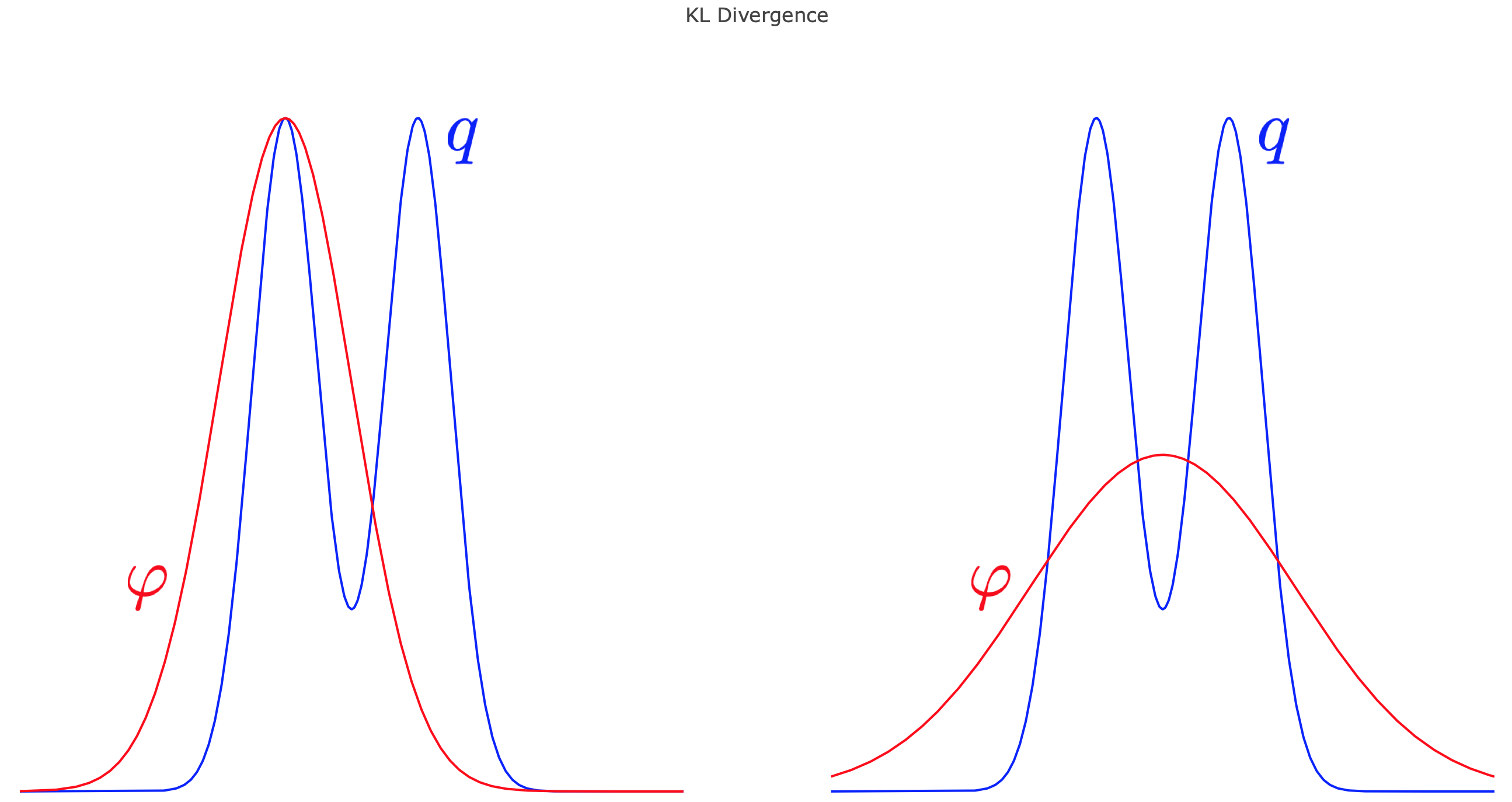}
\caption{Cartoon representation of the mode-seeking (left) and mean-seeking (right) Kullback-Leibler minimization. Mean-seeking minimization is employed within local-entropy optimization; mode-seeking minimization is employed within optimization of the heat-regularized loss. }
\label{KL_illustration}
\end{center}
\end{figure}

In the following two sections we show that, in addition to giving rather different solutions, the argument of the Kullback-Leibler divergence that is minimized has computational consequences.

\subsection{Minimization of $x \mapsto \dkl(q_{x_k,\tau}   \| \varphi_{x,\tau})$}\label{ssec:minalg1}
The solution to this minimization problem is unique and given by $\E_{X \sim q_{x_k,\tau} }\left( X \right).$
For notational convenience we drop the subscript $k$ and consider calculation of
\begin{equation}\label{eq:expectation}
\E_{X \sim q_{x,\tau}} \left( X \right).
\end{equation}
In our numerical experiments we will approximate these expectations using stochastic gradient Langevin dynamics and importance sampling. Both methods are reviewed in the next two subsections.

\subsubsection{Stochastic Gradient Langevin Dynamics}\label{ssec:SGLD}
The first method that we use to approximate the expectation \ref{eq:expectation} ---and thus the best-Gaussian approximation for local entropy optimization--- is stochastic gradient Langevin dynamics (SGLD). The algorithm was introduced in  \cite{welling2011bayesian} and  its  use for local entropy minimization was investigated in \cite{entropysgd}. The SGLD algorithm is summarized below.
\FloatBarrier
\begin{algorithm}
\begin{algorithmic}
\STATE Input: Sample size $J$ and temperatures $\{\epsilon_j\}_{j=1}^J.$ 
\begin{enumerate}
\item Define $x^0  = x. $
\item For $j=1, \ldots, J-1$ do:
$$x^{j+1} = x^j - \frac{\epsilon_j}{2}\Bigl(   \nabla f(x^j) - \frac{1}{\tau}(x-x^j) \Bigr) + \eta_j, \quad \quad \eta_t \sim N(0, \epsilon_j). $$
\end{enumerate}
Output: approximation $  \E_{X \sim q_{x,\tau}}(X) \approx \frac{\sum_{j=1}^J \epsilon_j x^j}{\sum_{j=1}^J \epsilon_j}. $
\end{algorithmic}
\caption{\label{algorithm3}}
\end{algorithm}
\FloatBarrier
When the function $f$ is defined by a large sum over training data, minibatches can be used in the evaluation of the gradients $\nabla f(x^j).$ In our numerical experiments we initialize the Langevin chain at the last iteration of the previous parameter update. Note that SGLD can be thought of as a modification of gradient-based Metropolis-Hastings Markov chain Monte Carlo algorithms, where the accept-reject mechanism is replaced by a suitable tempering of the temperatures $\epsilon_j.$

\subsubsection{Importance Sampling}\label{ssec:IS}
We will also investigate the use of importance sampling \cite{liu2008monte} to approximate the expectations \eqref{eq:expectation}; our main motivation in doing so is to avoid gradient computations, and hence to give an example of a training scheme that does not involve back propagation. 

Importance sampling is based on the observation that
\begin{equation*}
\E_{X \sim q_{x,\tau}}(X) = \int_{\R^d} x' q_{x,\tau}(x')\, dx' = \frac{ \int_{\R^d} x' \exp\bigl(-f(x') \bigr) \varphi_{x,\tau}(x') \, dx' }{ \int_{\R^d} \exp\bigl(-f(x') \bigr) \varphi_{x,\tau}(x') \, dx'},
\end{equation*}
and an approximation of the right-hand side may be obtained by standard Monte Carlo approximation of the numerator and  the denominator. Crucially, these Monte Carlo simulations are performed sampling the Gaussian $\varphi_{x,\tau}$ rather than the original density $q.$ The importance sampling algorithm is then given by:

\FloatBarrier
\begin{algorithm}
\begin{algorithmic}
\STATE Input: sample size $J.$
\begin{enumerate}
\item Sample $\{x^j\}_{j=1}^J$ from the Gaussian density $\varphi_{x,\tau}.$
\item Compute (unnormalized) weights $w^j = \exp\bigl(-f(x^j)\bigr).$
\end{enumerate}
\STATE Output: approximation 
\begin{equation}
 \E_{X \sim q_{x , \tau }}(X) \approx \frac{\sum_{j=1}^J w^j x^j}{ \sum_{j=1}^J w^j}. 
\label{IS}
\end{equation}
\end{algorithmic}
\caption{\label{algorithm4}}
\end{algorithm}
\FloatBarrier

Importance sampling is easily parallelizable. If $L$ processors are available, then each of the processors can be used to produce an estimate using $J/L$ Gaussian samples, and the associated estimates can be subsequently consolidated. 

While the use of importance sampling opens an avenue for gradient-free, parallelizable training of neural networks, our numerical experiments will show that naive implementation without parallelization gives poor performance relative to SGLD or plain stochastic gradient descent (SGD) descent on the original loss. A potential explanation is the so-called curse of dimension for importance sampling \cite{sanz2016importance}, \cite{agapiou2015importance}. Another explanation is that the iterative structure of SGLD allows to re-utilize the previous parameter update to approximate the following one while importance sampling does not afford such iterative updating. Finally, SGLD with minibatches is known to asymptotically produce unbiased estimates, while the introduction of minibatches in importance sampling introduces a bias.

\subsection{Minimization of $x \mapsto \dkl(  \varphi_{x,\tau} \|q_{x_k,\tau}  )$}\label{ssec:minalg2}
A direct calculation shows that the preconditioned Euler-Lagrange equation for minimizing $x \mapsto \dkl(  \varphi_{x,\tau} \|q_{x_k,\tau}  )$ is given by
\begin{equation*}\label{eq:Euler-Lagrange}
h(x) :=x-x_k +\tau \, \E_{Y\sim \varphi_{x,\tau}} \nabla f(Y) = 0.
\end{equation*}
Here $h(x)$ is implicitly defined as an expected value with respect to a distribution that depends on the parameter $x.$ The Robbins-Monro algorithm \cite{robbins1956empirical} allows to estimate zeroes of functions defined in such a way.
\FloatBarrier
\begin{algorithm}
\begin{algorithmic}
\STATE Input: Number of iterations $J$ and schedule $\{a^j\}_{j=1}^J.$ 
\begin{enumerate}
\item Define $x^0  = x. $
\item For $j=1, \ldots, J$ do:
\begin{equation}\label{eq:robmon}
x^{j+1} = x^j - a^j \left\{ x^j - x_k + \frac{\tau}{M} \sum_{m=1}^M \nabla f\bigl(z^{(m)}\bigr)  \right\}, \quad \quad z^{(m)} \sim \varphi_{x^j, \tau}.
\end{equation}
\end{enumerate}
\caption{\label{Robinsalgorithm}}
Output: approximation $  x^J$ to the minimizer of $x \mapsto \dkl(  \varphi_{x,\tau} \|q_{x_k,\tau}  ).$
\end{algorithmic}
\end{algorithm}
\FloatBarrier

The Robbins-Monro approach to computing the Gaussian approximation $(x,\tau) \mapsto \dkl(  \varphi_{x,\tau} \|q_{x_k,\tau}  )$ in Hilbert space was studied in \cite{pinski2015algorithms}. A suitable choice for the step size is $a^l = c l^\alpha$,  for some $c>0$ and $\alpha \in (1/2, 1].$ Note that Algorithm \ref{Robinsalgorithm} gives a form of {\emph spatially}-averaged gradient descent, which involves repeated evaluation of  the gradient of the original loss. The use of \emph{temporal} gradient averages has also been studied as a way to reduce the noise level of stochastic gradient methods \cite{bottou2016optimization}. 

To conclude we remark that an alternative approach could be to employ Robbins-Monro  directly to optimize $F^H(x).$ Gradient calculations would still be needed.

\section{Numerical Experiments}\label{sec:experiments}
In the following numerical experiments we investigate the practical use of local entropy and heat regularization in the training of neural networks. We present experiments on dense multilayered networks applied to a basic image classification task, viz. MNIST \cite{MNISTLeCun}.  We implement Algorithms \ref{algorithm3}, \ref{algorithm4}, and \ref{Robinsalgorithm} in TensorFlow, analyzing the effectiveness of each in comparison to stochastic gradient descent (SGD). We investigate whether the theoretical monotonicity of regularized losses translates into monotonicity of the held-out test data error.
\nc Additionally, we explore various choices for the hyper-parameter $\tau$ to illustrate the effects of variable levels of regularization. In accordance to the algorithms specified above, we employ importance sampling (IS) and stochastic gradient Langevin dynamics (SGLD) to approximate the expectation in \eqref{eq:expectation}, and the Robbins-Monro algorithm for heat regularization (HR).

\subsection{Network Specification}

Our experiments are carried out using the following networks:

\begin{enumerate}

\item Small Dense Network: Consisting of an input layer with 784 units and a 10 unit output layer, this toy network contains 7850 total parameters and achieves a test accuracy of 91.2 \% when trained with SGD for 5 epochs over the 60,000 image MNIST dataset. 

\item Single Hidden Layer Dense Network: Using the same input and output layer as the smaller network with an additional 200 unit hidden layer, this network provides an architecture with 159,010 parameters. We expect this architecture to achieve a best-case performance of  98.9 \%  accuracy on MNIST, trained over the same data as the previous network.

\end{enumerate}

\subsection{Training Neural Networks From Random Initialization}

Considering the computational burden of computing a Monte Carlo estimate for each weight update, we propose that Algorithms \ref{algorithm3}, \ref{algorithm4}, and \ref{Robinsalgorithm} are potentially most useful when employed following SGD; although per-update progress is on par or exceeds that of SGD with step size, often called learning rate, equivalent to the value of $\tau$, the computational load required makes the method unsuited for end-to-end training. Though in this section we present an analysis of these algorithms used for the entirety of training, this approach is likely too expensive to be practical for contemporary deep networks.  

\begin{table}
\caption{Classification Accuracy on Held-Out Test Data}
\centering
\begin{tabular}{c c c c c c}
\hline \hline
Weight Updates & 100 & 200  & 300 & 400 & 500  \\
\hline
SGD & 0.75 & 0.80 & 0.85 & 0.87 & 0.87 \\
IS & 0.27 & 0.45 & 0.54 & 0.57 & 0.65 \\
SGLD &0.72 & 0.81 & 0.84 & 0.86 & 0.88 \\
HR & 0.52 & 0.64 & 0.70 & 0.73 & 0.76  \\
\hline
\label{table:Held-out_Test_Accuracy}
\end{tabular}
\end{table}

\begin{figure}
\includegraphics[width=\linewidth]{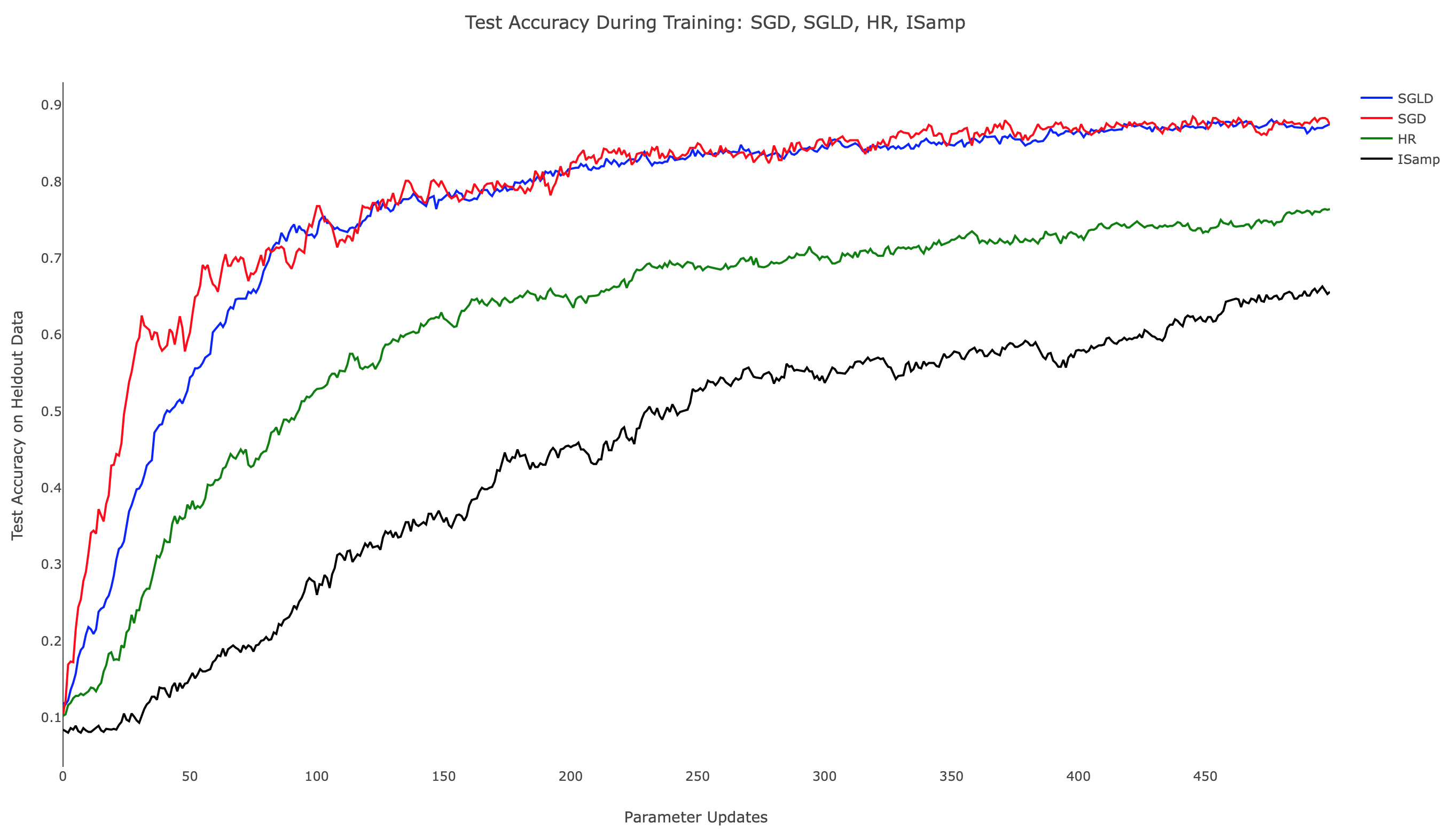}
\caption{Held-out Test Accuracy during training for SGD (Red), SGLD (Blue), HR (Green), and IS (Black). $\tau = 0.01$ for SGLD, IS, and HR. Learning rate of SGD is also $0.01$. SGLD uses temperatures $\epsilon_j = {\frac{1}{1000 + j}}$ and HR's update schedule uses $c = 0.1,$ and $ \alpha = 0.7$.}
\label{fig:CompTrain}
\end{figure}

Table \ref{table:Held-out_Test_Accuracy} and the associated Figure \ref{fig:CompTrain} demonstrate the comparative training behavior for each algorithm, displaying the held-out test accuracy for identical instantiations of the hidden layer network trained with each algorithm for 500 parameter updates. Note that a mini-batch size of 20 was used in each case to standardize the amount of training data available to the methods. Additionally, SGLD, IS, and HR each employed $\tau = 0.01$, while SGD utilized an equivalent step size, thus fixing the level of regularization in training. To establish computational equivalence between Algorithms \ref{algorithm3}, \ref{algorithm4}, and \ref{Robinsalgorithm}, we compute $\E_{X \sim q} \left( X \right)$ with $10^3$ samples for Algorithms \ref{algorithm3} and \ref{algorithm4}, setting $M = 30$ and performing $30$ updates of the chain in Algorithm \ref{Robinsalgorithm}. Testing accuracy was computed by classifying 1000 randomly selected images from the held-out MNIST test set. In related experiments, we observed consistent training progress across all three algorithms. In contrast, IS and HR trained more slowly, particularly during the parameter updates following initialization.  From Figure \ref{fig:CompTrain} we can appreciate that while SGD attempts to minimize training error, it nonetheless behaves in a stable way when plotting held-out accuracy, specially towards the end of training. SGLD on the other hand is observed to be more stable throughout the whole training.

While SGD, SGLD, and HR utilize gradient information in performing parameter updates, IS does not. This difference in approach contributes to IS's comparatively poor start; as the other methods advance quickly due to the large gradient of the loss landscape, IS's progress is isolated, leading to training that depends only on the choice of $\tau$. When $\tau$ is held constant, as shown in ~\ref{fig:CompTrain}, the rate of improvement remains nearly constant throughout. This suggests the need for dynamically updating $\tau$, as is commonly performed with annealed learning rates for SGD. Moreover, SGD, SGLD and HR are all schemes that depend linearly in $f$, making mini-batching justifiable, something that is not true for IS. 

\begin{table}
\caption{Runtime Per Weight Update}
\centering
\begin{tabular}{c c}
\hline \hline
& Average Update Runtime (Seconds) \\
\hline
SGD & 0.0032 \\
IS & 6.2504  \\
SGLD & 7.0599  \\
HR & 3.3053   \\
\hline
\label{table:Update Runtime}
\end{tabular}
\end{table}

It is worth noting that the time to train differed drastically between methods. Table \ref{table:Update Runtime} shows the average runtime of each algorithm in seconds. SGD performs roughly $10^3$ times faster than the others, an expected result considering the most costly operation in training, filling the network weights, is performed $10^3$ times per parameter update. Other factors contributing to the runtime discrepancy are the implementation specifications and the deep learning library; here, we use TensorFlow's implementation of SGD, a method for which the framework is optimized. More generally, the runtimes in Table \ref{table:Update Runtime} reflect the hyper-parameter choices for the number of Monte Carlo samples, and will vary according to the number of samples considered.

\subsection{Local Entropy Regularization after SGD}

\begin{figure}
\includegraphics[width=\linewidth]{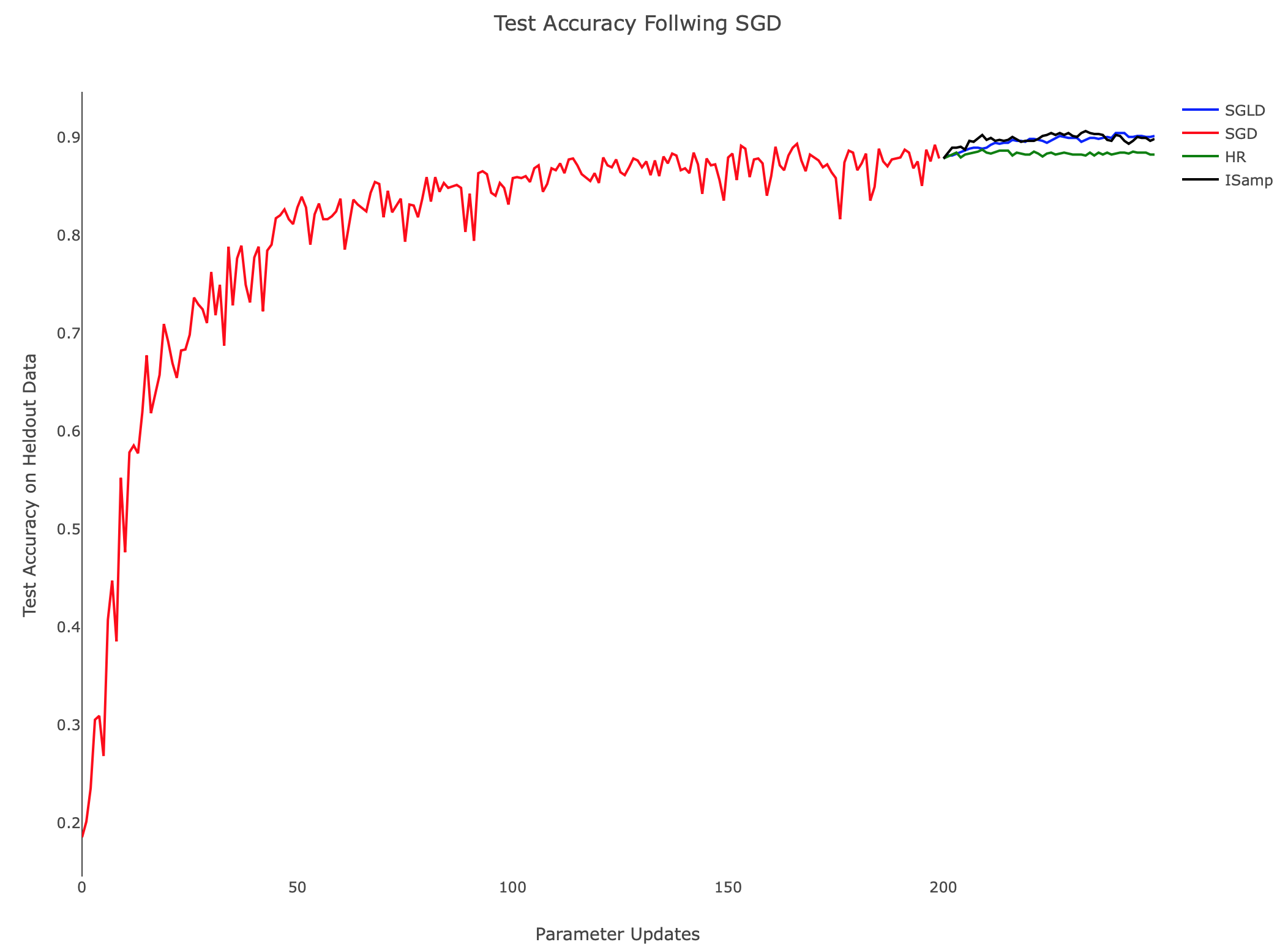}
\caption{Training after SGD, $\tau = 0.01$ for all algorithms. The step size for the SGD is set equal to the value of $\tau$ for all three algorithms. SGLD temperatures are $\epsilon_j = {\frac{1}{2000 + j}}$, and HR uses the same update schedule as in Figure \ref{fig:CompTrain}.}
\label{fig:AfterSGD}
\end{figure}

Considering the longer runtime of the sampling based algorithms in comparison to SGD, it is appealing to utilize SGD to train networks initially, then shift to more computationally intensive methods to identify local minima with favorable generalization properties.  
Figure \ref{fig:AfterSGD} illustrates IS and SGLD performing better than HR when applied after SGD. HR's smooths the loss landscape, a transformation which is advantageous for generating large steps early in training, but presents challenges as smaller features are lost. In Figure \ref{fig:AfterSGD}, this effect manifests as constant test accuracy after SGD, and no additional progress is made. The contrast between each method is notable since the algorithms use equivalent step sizes---this suggests that the methods, not the hyper-parameter choices, dictate the behavior observed.

Presumably, SGD trains the network into a sharp local minima or saddle point of the non-regularized loss landscape; transitioning to an algorithm which minimizes the local entropy regularized loss then finds an extrema which performs better on the test data. However, based on our experiments, in terms of held-out data accuracy regularization in the later stages does not seem to provide significant improvement over training with SGD on the original loss. \nc


\subsection{Algorithm Stability \& Monotonicity}

\begin{figure}
\includegraphics[width=\linewidth]{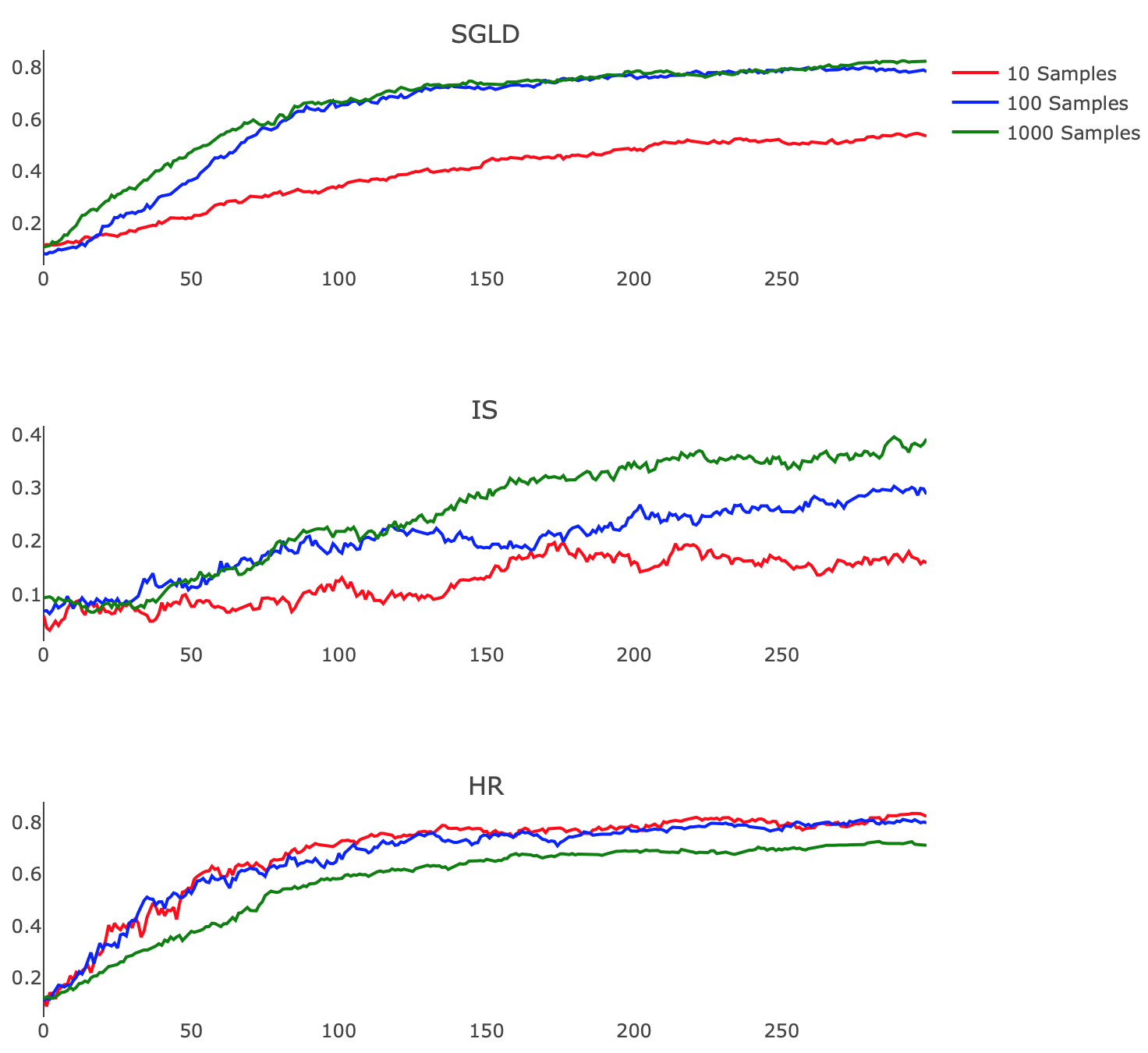}
\caption{Training behaviors with $S_i = \{10 \text{ (Red)}, \  100 \text{ (Blue)}, \ 1000 \text{ (Green)} \}$ samples per parameter update. SGLD temperatures and HR schedule are the same as in Figure \ref{fig:CompTrain}. Note that $\tau = 0.01$ throughout. To equalize computational load across algorithms, we set $M = J = \lfloor \sqrt{S_i} \rfloor $ for HR.}
\label{fig:AlgStability}
\end{figure}

Prompted by the guarantees of theorems \ref{theoremmonotone} and \ref{theoremmonotone2} which prove the effectiveness of these methods when $\E_{X \sim q} \left( X \right)$ is approximated accurately, we also demonstrate the stability of these algorithms in the case of an inaccurate estimate of the expectation. To do so, we explore the empirical consequences of varying the number of samples used in the Monte Carlo and Robbins-Monro calculations. 

Figure \ref{fig:AlgStability} shows how each algorithm responds to this change. We observe that IS performs better as we refine our estimate of $\E_{X \sim q} \left( X \right)$, exhibiting less noise and faster training rates. This finding suggests that a highly parallel implementation of IS which leverages modern GPU architecture to efficiently compute the relevant expectation may offer practicality. SGLD also benefits from a more accurate approximation, displaying faster convergence and higher final testing accuracy when comparing 10 and 100 Monte Carlo samples. HR however performs more poorly when we employ longer Robbins-Monro chains, suffering from diminished step size and exchanging quickly realized progress for less oscillatory testing accuracy. Exploration of the choices of $\epsilon_j$ and $a^j$ for SGLD and HR remains a valuable avenue for future research, specifically in regards to the interplay between these hyper-parameters and the variable accuracy of estimating $\E_{X \sim q} \left( X \right)$.

\subsection{Choosing $\tau$}

\begin{figure}
\includegraphics[width=\linewidth]{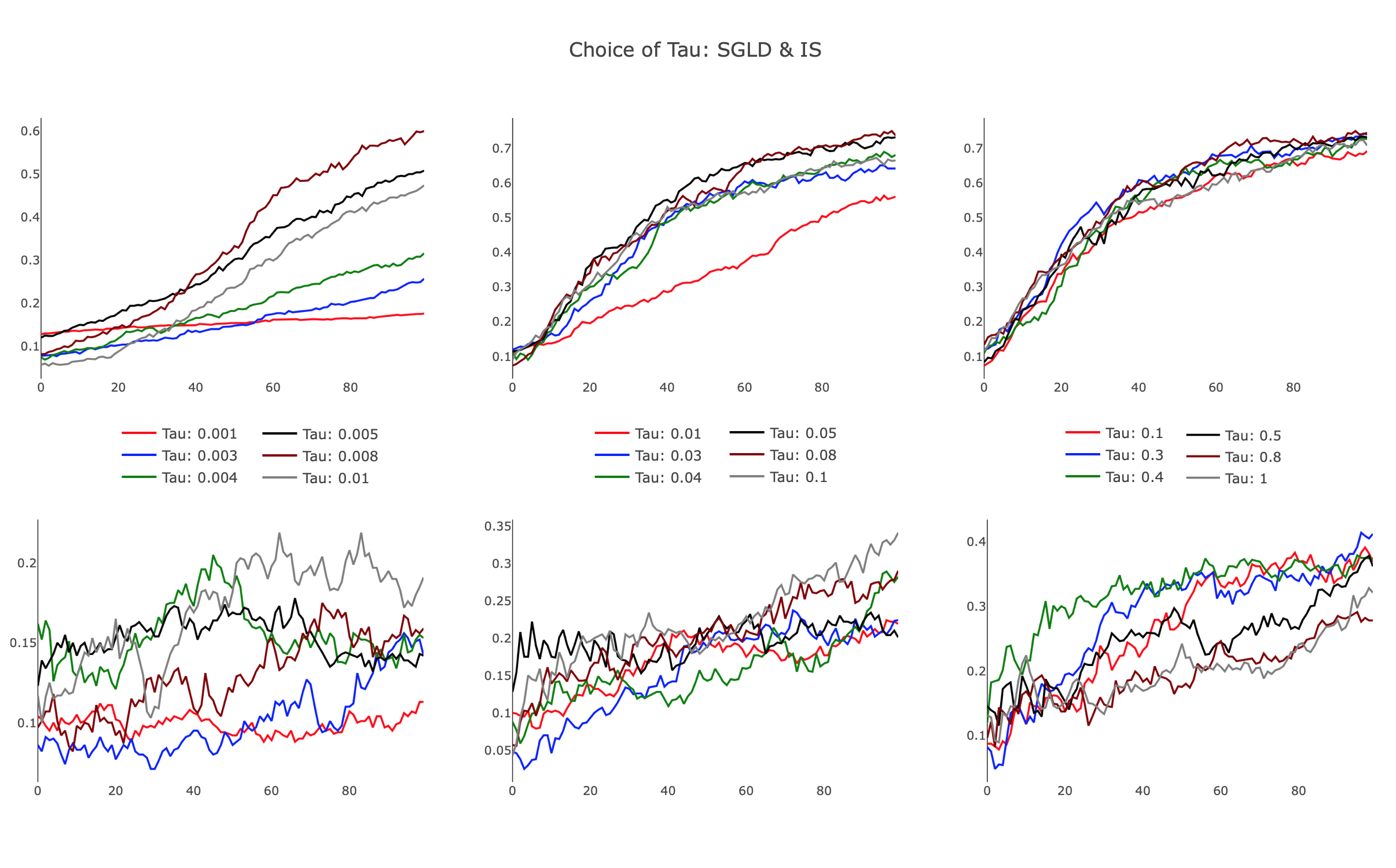}
\caption{Training the smaller neural network with different choices for $\tau$ using SGLD and IS. Values of $\tau$ vary horizontally from very small to large: $\tau \in \{[0.001, 0.01], [0.01, 0.1], [0.1, 1] \}$. Top row shows SGLD with $\epsilon_j = {\frac{1}{1000 + j}}$, bottom row shows IS. All network parameters were initialized randomly.}
\label{fig:ChoosingTau}
\end{figure}

An additional consideration of these schemes is the choice of $\tau$, the hyper-parameter which dictates the level of regularization in Algorithms \ref{algorithm3}, \ref{algorithm4}, and \ref{Robinsalgorithm}. As noted in \cite{entropysgd}, large values of $\tau$ correspond to a nearly uniform local entropy regularized loss, whereas small values of $\tau$ yield a minimally regularized loss which is very similar the original loss function. To explore the effects of small and large values of $\tau$, we train our smaller network with IS and SGLD for many choices of $\tau$, observing how regularization alters training rates.

The results, presented in Figure \ref{fig:ChoosingTau}, illustrate differences in SGLD and IS, particularly in the small $\tau$ regime. As evidenced in the leftmost plots, SGLD trains successfully, albeit slowly, with $\tau \in [0.001, 0.01]$. For small values of $\tau$, the held-out test accuracy improves almost linearly over parameter updates, appearing characteristically similar to SGD with a small learning rate. IS fails for small $\tau$, with highly variant test accuracy improving only slightly during training. Increasing $\tau$, we observe SGLD reach a point of saturation, as additional increases in $\tau$ do not affect the training trajectory. We note that this behavior persists as $\tau \to \infty$, recognizing that the regularization term in the SGLD algorithm approaches a value of zero for growing $\tau$. IS demonstrates improved training efficiency in the bottom-center panel, showing that increased $\tau$ provides favorable algorithmic improvements. This trend dissipates for larger $\tau$, with IS performing poorly as $\tau \to \infty$. The observed behavior suggests there exists an optimal $\tau$ which is architecture and task specific, opening opportunities to further develop a heuristic to tune the hyper-parameter.

\subsubsection{Scoping of $\tau$}

As suggested in \cite{entropysgd}, we anneal the scope of $\tau$ from large to small values in order to examine the landscape of the loss function at different scales. Early in training, we use comparatively large values to ensure broad exploration, transitioning to smaller values for a comprehensive survey of the landscape surrounding a minima. We use the following schedule for the $k^{th}$ parameter update:

\begin{equation*}
\tau (k) = \frac{\tau_0}{(1 + \tau_1)^k} 
\label{eq:ScopeSchedule}
\end{equation*}

where $\tau_0$ is large and $\tau_1$ is set so that the magnitude of the local entropy gradient is roughly equivalent to that of SGD.  

\begin{figure}
\includegraphics[width=\linewidth]{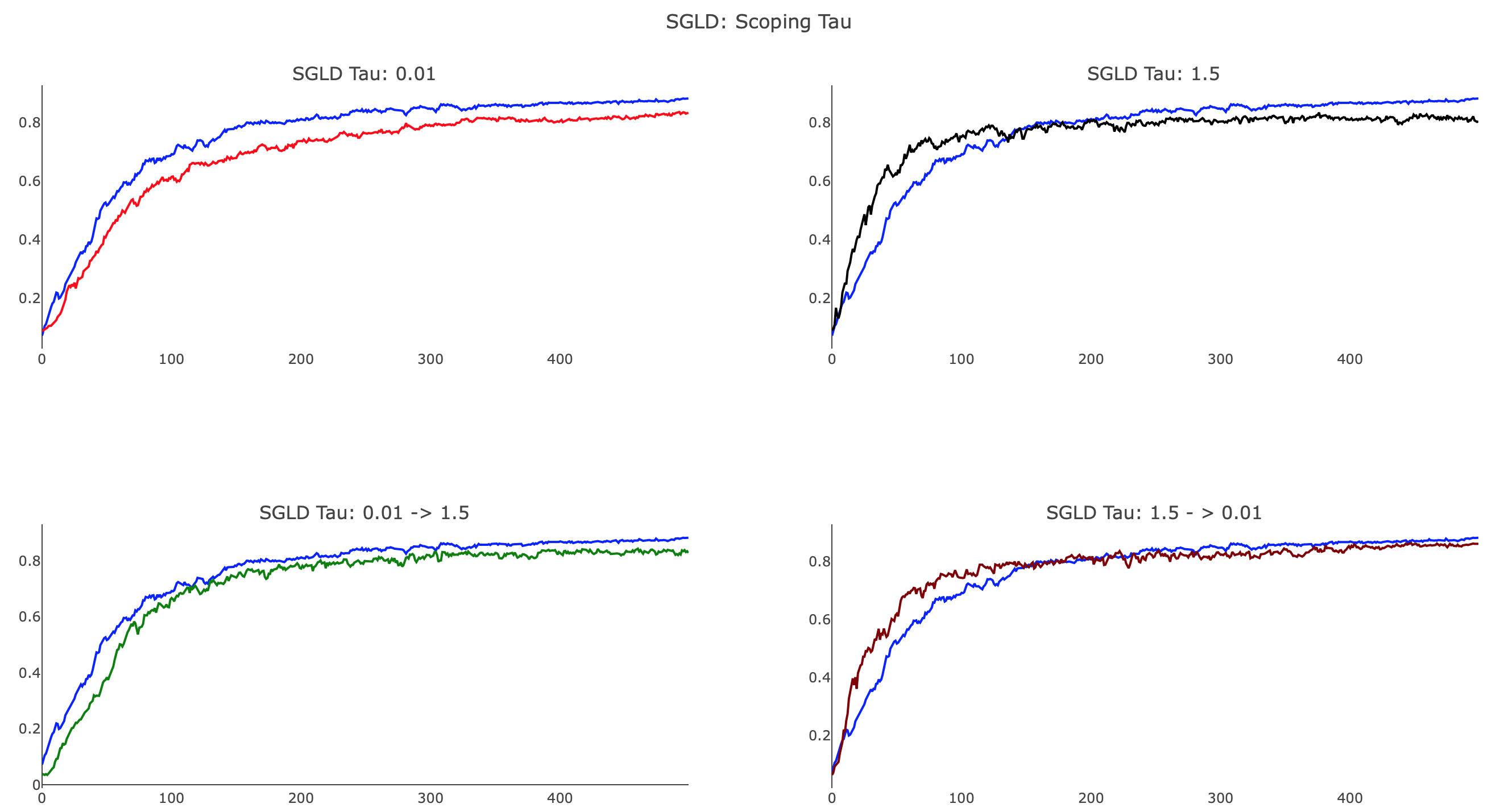}
\caption{Examination of the effects of scoping $\tau$ during training via the update schedule \ref{eq:ScopeSchedule}. All four panels display SGLD with temperatures set as $\epsilon_j = {\frac{1}{1000 + j}}$, and SGD (Blue) with a learning rate of $0.01$. Top: SGLD with constant $\tau$, set as $\tau = 0.01$ and $\tau = 1.5$. Bottom: $\tau$ scoped as $\tau: 0.01 \to 1.5$ and $\tau: 1.5 \to 0.001$.}
\label{fig:MultiPanelTauScoping}
\end{figure}

As shown in Figure \ref{fig:MultiPanelTauScoping}, annealing $\tau$ proves to be useful, and provides a method by which training can focus on more localized features to improve test accuracy. We observe that SGLD, with a smaller value of $\tau = 0.01$, achieves a final test accuracy close to that of SGD, whereas $\tau = 1.5$ is unable to identify the optimal minima. Additionally, the plot shows that large $\tau$ SGLD trains faster than SGD in the initial 100 parameter updates, whereas small $\tau$ SGLD lags behind. When scoping $\tau$ we consider both annealing and reverse-annealing, illustrating that increasing $\tau$ over training produces a network which trains more slowly than SGD and is unable to achieve testing accuracy comparable to that of SGD. Scoping $\tau$ from $1.5 \to 0.01$ via the schedule \ref{eq:ScopeSchedule} with $\tau_0 = 1.5$ and $\tau_1 = 0.01$ delivers advantageous results, yielding an algorithm which trains faster than SGD after initialization and achieves analogous testing accuracy.

\bibliographystyle{plain}
\bibliography{isbib}

\end{document}